\documentclass[a4paper]{article}

\usepackage{float}
\usepackage{amsmath,amssymb,bm}
\usepackage{amsthm}
\usepackage{ascmac}
\usepackage{array}
\usepackage{enumerate}
\usepackage{mathtools}
\usepackage{multirow}

\newtheorem{theorem}{Theorem}[section]
\newtheorem*{theorem*}{Theorem}
\newtheorem{definition}[theorem]{Definition}
\newtheorem*{definition*}{Definition}
\newtheorem{lemma}{Lemma}[section]
\newtheorem{corollary}{Corollary}[section]

\newcommand{\argmin}{\mathop{\rm argmin}\limits}

\title{Upper Bound of Real Log Canonical Threshold of Tensor Decomposition and its Application to Bayesian Inference}
\author{Naoki Yoshida and Sumio Watanabe\\
Department of Mathematical and Computing Science\\
Tokyo Institute of Technology\\
Mail-Box W8-42, 2-12-1, Oookayama, Meguro-ku, Tokyo,\\
152-8552, Japan}

\date{}

\begin{document}

\maketitle

\begin{abstract}
Tensor decomposition is now being used for data analysis, information compression, and
knowledge recovery. However, the mathematical property of tensor decomposition is 
not yet fully clarified because it is one of singular learning machines. 
 In this paper, we give the upper bound of its real log canonical threshold (RLCT) of
 the tensor decomposition by using an algebraic geometrical method and derive its Bayesian generalization error theoretically. We also give considerations about its mathematical property through numerical experiments.
\end{abstract}

\section{Introduction}
Tensor decomposition is widely used in data science and machine learning \cite{1}. 
For instance, It plays the central roles in 
signal processing by contribution analysis \cite{2}, data compression by converting tensor data to matrix data \cite{3}, and data recovery by counting backwards from the matrices to the original tensor data \cite{4}. In many cases, tensor decomposition itself is known to be NP-hard \cite{21}. For this reason, tensor decomposition is often calculated approximately by Bayesian inference. However, its mathematical property is not yet completely clarified because
it is one of the singular statistical models. In this paper, we derive its generalization performance 
in Bayesian inference. 

Tensor decomposition has mainly two types: Tucker decomposition and CP decomposition. For three-dimensional tensor $X\in \mathbb{R}^{IJK}$, Tucker decomposition is defined by
\begin{equation*}
X_{ijk} = \sum_{h_1=1}^{H_1}\sum_{h_2=1}^{H_2}\sum_{h_3=1}^{H_3} \Lambda_{h_1h_2h_3}A_{ih_1}B_{jh_2}C_{kh_3},
\end{equation*}
where $A\in \mathbb{R}^{NH_1}$, $B\in \mathbb{R}^{NH_2}$, $C\in \mathbb{R}^{NH_3}$, and $\Lambda \in \mathbb{R}^{H_1H_2H_3}$. $A$, $B$, and $C$ are called factor matrices, and $\Lambda$ is called a core tensor. CP decomposition is a special case of Tucker decomposition: it is defined by
\begin{equation*}
X_{ijk}=\sum_{h=1}^H \Lambda_h A_{ih}B_{jh}C_{kh},
\end{equation*}
where $A\in \mathbb{R}^{IH}$, $B\in \mathbb{R}^{JH}$,
$C\in \mathbb{R}^{KH}$, and $\Lambda \in \mathbb{R}^{H}$.

We study a statistical model of a tensor decomposition
\begin{equation}
\label{1}
p(X|A, B, C) = \frac{1}{(2\pi)^{IJK/2}} \exp\left(-\frac{1}{2}\|X - A\circ B\circ C\|^2\right),
\end{equation}
 where $X \in \mathbb{R}^{I \times J \times K}$, $A \in \mathbb{R}^{I \times H}$, $B \in \mathbb{R}^{J \times H}$, and $C \in \mathbb{R}^{K \times H}$, and 
 $\|\;\;\|^2$ is the squared Frobenius norm, i.e., sum of squared components.  Here
 $A\circ B\circ C$ is a $I \times J \times K$ tensor, whose $(i,j,k)$ component is defined by
\begin{equation*}
(A\circ B\circ C)_{ijk} = \sum_{h=1}^H A_{ih}B_{jh}C_{kh}.
\end{equation*}
Note that the operator $\circ$ is not associative.  
The model, \eqref{1}, is nonidentifiable and singular
because the map from the parameter $(A,B,C)$ to the probability distribution $p(x|A,B,C)$
is not one-to-one. 

In statistical learning theory, the generalization performance is one of the most important properties of a statistical model. Let $X^n=(X_1,X_2,...,X_n)$ be a set of independent random variables whose probability distribution is $q(x)$, and $p^*(x|X^n)$ be the Bayesian predictive distribution of $x$ for
a given $X^n$ and a statistical model and a prior distribution. Then the generalization error
$G_n$ is defined by the Kullback-Leibler divergence from $q(x)$ to $p^*(x|X^n)$, 
\begin{equation*}
G_n = KL(q(x)||p^*(x|X^n)). 
\end{equation*}
If the model is regular and if a data-generating distribution is realizable by a statistical model, 
the
expected value of $G_n$ over training data is given by 
\begin{equation*}
\mathbb{E}[G_n] = \frac{m}{2n} + o(\frac{1}{n}), 
\end{equation*}
where $m$ is the number of parameters \cite{5}.
If a model is singular, then Fisher information matrix is degenerate, 
hence we cannot employ information criteria such as AIC\cite{13}, BIC\cite{14}, or TIC\cite{15}.
 However, even for a singular model, the  asymptotic behavior of $\mathbb{E}[G_n]$ is obtained by the resolution of singularities \cite{6,7}:
\begin{equation}
\label{2}
\mathbb{E}[G_n] = S + \frac{\lambda}{n} + o(\frac{1}{n}),
\end{equation}
where $S=\int q(x)\log \frac{q(x)}{p(x|w_0)}$, $w_0=\argmin_{w} KL(q(x)||p(x|w))$, which are constant values determined by the ability of expression of the model, and $\lambda$ is the real log canonical threshold (RLCT) whose definition is given in Definition \ref{Definition:RLCT} in the next section. Generally, RLCT is not larger than half the number of parameters, which indicates a singular model makes the generalization error smaller than that of a regular model.

RLCT is a birational invariant determined by the model and the prior distribution. 
For several statistical models and learning machines, RLCTs or their upper bounds 
have been clarified,
for example, three-layered neural networks\cite{8}, 
reduced rank regression and matrix factorization \cite{9}, 
normal mixtures\cite{10}, 
multinomial mixtures\cite{11}, and 
the Latent Dirichlet allocation(LDA) \cite{16}, 
hidden Markov models\cite{12}, and  nonnegative matrix factorization \cite{17}.
However, 
RLCT of  a tensor decomposition model has been left unknown, and 
in this paper, we theoretically derive the upper bound of RLCT of a tensor decomposition model
and discuss its tightness. 

This paper consists of four parts. In the second section, the upper bound of the RLCT of the tensor decomposition model is shown in the main theorem. In the third section, the main theorem is proved. In the fourth section, we examine the tightness of the upper bound, comparing with the result in 
 numerical experiments.
 
 \section{Settings and Main Result}

We assume the data $X$ is \textit{i.i.d.} whose probability density function is $q(x)$. Let $X^n=\{X_i\}_{i=1}^n$, $\phi(x)$, and $p(x|w)$ be
a training data, a prior distribution of a parameter $w$, and  a statistical 
model or a learning machine, respectively.  
The posterior distribution $p(w|X^n)$ is defined by 
\begin{equation*}
p(w|X^n)=\frac{1}{Z_n}\phi(w)\prod_{i=1}^n p(X_i|w),
\end{equation*}
where $Z_n$ is the normalizing constant.
The Bayes predictive distribution $p^*$ is defined by
\begin{equation*}
p^*(x|X^n) = \int p(x|w)p(w|X^n)dw. 
\end{equation*}
Then RLCT is defined by the following way. 
\begin{definition}[RLCT of a statistical model]\label{Definition:RLCT}
Let $K(w)$ be the Kullback-Leibler divergence from $q(x)$ to $p(x|w)$, i.e.
$K(w)=\int q(x)\log (q(x)/p(x|w))dx$. Then the zeta function defined in $\Re(z)>0$ 
\begin{equation*}
\zeta(z)=\int K(w)^z\phi(w)dw
\end{equation*}
can be analytically continued to the unique meromorphic function on the 
entire complex plane, whose poles are all real and negative. 
Let $-\lambda<0$ be the maximum pole of the zeta function,
then $\lambda>0$ is the RLCT of the model, which is denoted by $\lambda(K, \phi)$.
\end{definition}

A tensor decomposition model is expressed as eq.\eqref{1}. 
We assume that there exists a parameter $w_0 = (A_0, B_0, C_0)$ which satisfies $p(x|w_0)=q(x)$,
where
 $A_0\in\mathbb{R}^{IH_0}$, $B_0\in\mathbb{R}^{JH_0}$, $C_0\in\mathbb{R}^{KH_0}$, 
 and $H_0\leq H$. Also it is assumed that $\phi(A, B, C)$ is bounded by a 
 $C^\infty$-function whose support contains $w_0$. 
 Under these assumptions, we obtain the following theorem. 

\begin{theorem}[Main Theorem]
The RLCT of the tensor decomposition model satisfies the following inequality:
\begin{equation*}
\lambda \leq \frac{H_0(I+J+K)-2}{2}+\min\{m_1, m_2, m_3\},
\end{equation*}
where
\begin{align*}
&m_1 = F(IJ, K, H-H_0), \\
&m_2 = F(JK, I, H-H_0), \\
&m_3 = F(KI, J, H-H_0),
\end{align*}
and
\begin{align*}
&F(N, M, H)= \\
&\begin{cases}
{\frac{2(NM+MH+HN)-(N^2+M^2+H^2)}{8}(|N-M|\leq H\leq N+M\,and\,H+M-N\,is\,even)} \\
{\frac{2(NM+MH+HN)-(N^2+M^2+H^2)+1}{8} (|N-M|\leq H\leq N+M\,and\,H+M-N\,is\,odd)} \\
{\frac{MH}{2}(H\leq N-M)} \\
{\frac{NH}{2}(H\leq M-N)} \\
{\frac{NM}{2}(M+N\leq H)}. \\
\end{cases}
\end{align*}

\end{theorem}

\section{Proof of the Main Theorem}

\subsection{Preliminary}

As the general properties of the RLCT, the following lemma holds \cite{18}.
\begin{lemma}
\label{lem.1}
Let $K_1(w)$, $K_2(w)$ be nonnegative analytic functions and $\phi_1(w)$, $\phi_2(w)$ be 
nonnegative $C^\infty$-functions with compact support.

\begin{enumerate}[(a)]
\item (Inequality) If $K_1(w)\leq K_2(w)$ and $\phi_1(w)\geq \phi_2(w)$ then $\lambda(K_1, \phi_1)\leq \lambda(K_2, \phi_2)$.

Especially, If $\phi_1(w)=\phi_2(w)$ and $c_1K_1(w)\leq K_2(w)\leq c_2K_1(w)$, where $c_1$ and $c_2$ are positive constants, then $\lambda(K_1, \phi_1)=\lambda(K_2, \phi_2)$.

\item (Product) If $w=(w_1, w_2)$, $K(w)=K_1(w_1)K_2(w_2)$, $\phi(w)=\phi_1(w_1)\phi_2(w_2)$ then $\lambda(K,\phi)=\min \{ \lambda(K_1, \phi_1), \lambda(K_2, \phi_2)\}$.

\item (Sum) If $w=(w_1, w_2)$, $K(w)=K_1(w_1)+K_2(w_2)$, $\phi(w)=\phi_1(w_1)\phi_2(w_2)$ then $\lambda(K,\phi)=\lambda(K_1, \phi_1)+ \lambda(K_2, \phi_2)$.
\end{enumerate}
\end{lemma}

Next, we have $K(w)$ of the tensor decomposition model.
\begin{lemma}
In regards to tensor decomposition model, $K(w)$ is expressed as follows:
\begin{equation*}
K(w)=\|A\circ B \circ C-A_0\circ B_0\circ C_0\|^2.
\end{equation*}
\end{lemma}

\begin{proof}
Let $T=A\circ B \circ C$, $T_0=A_0\circ B_0\circ C_0$.
Then, 
\begin{align*}
K(w)&=\int q(x)\log{\frac{p(x|A_0,B_0,C_0)}{p(x|A,B,C)}}dx \\
&=\int q(x)\left\{ \frac{1}{2}\left(\|x-T\|^2-\|x-T_0\|^2\right)\right\}dx \\
&=\frac{1}{2}\left(\|T\|^2-\|T_0\|^2\right)+\int \sum_{i,j,k} \{(T_0)_{ijk}-T_{ijk}\}x_{ijk}q(x)dx \\
&=\frac{1}{2}\left(\|T\|^2-\|T_0\|^2\right)+\sum_{i,j,k} \{(T_0)_{ijk}-T_{ijk}\}(T_0)_{ijk} \\
&=\frac{1}{2}\|T-T_0\|^2.
\end{align*}
Using Lemma \ref{lem.1}(a) completes the proof.
\end{proof}

Let us define an equivalence relation $f(w)\sim g(w)$ fot the case when 
 $c_1f(w)\leq g(w)\leq c_2f(w)$ holds, where $c_1$ and $c_2$ are positive constants. 
 We can divide the upper bound of $\lambda$ into two parts by the following inequality:
\begin{align*}
&\|A\circ B\circ C - A_0\circ B_0\circ C_0\|^2 = \sum_{i,j,k}\left(\sum_{h=1}^H A_{ih}B_{jh}C_{kh}-\sum_{h=1}^{H_0} (A_0)_{ih}(B_0)_{jh}(C_0)_{kh}\right)^2 \\
\leq &2\sum_{i,j,k}\left(\sum_{h=1}^{H_0}\left(A_{ih}B_{jh}C_{kh}-(A_0)_{ih}(B_0)_{jh}(C_0)_{kh}\right)\right)^2
\\
& +2\sum_{i,j,k}\left(\sum_{h=H_0+1}^H A_{ih}B_{jh}C_{kh}\right)^2.
\end{align*}
The variables are separated 
between the first and second term of the right side of the inequality, so Lemma \ref{lem.1}(c) can be applied. Denoting $\lambda_1$, $\lambda_2$ as the RLCT of the first and second term, $\lambda_1+\lambda_2$ is the upper bound of $\lambda$.
We prepare for obtaining the upper bound of $\lambda_1$ by subsection \ref{section1} and that of $\lambda_2$ by subsection \ref{section2}, and complete the proof of the main theorem in subsection \ref{section3}.

\subsection{Preparing for Upper Bound of $\lambda_1$}
\label{section1}
First, let us prepare for getting the upper bound of $\lambda_1$. By the following two lemmas, $A_0$, $B_0$, $C_0$ can be changed to diagonal matrices, which simplifies the problem.
We write $(A)_{ih}(1\leq h\leq H_0)$, $(B)_{jh}(1\leq h\leq H_0)$, $(C)_{kh}(1\leq h\leq H_0)$ by $A$, $B$, $C$ as it is.
\begin{lemma}
\label{lem.2}
Let arbitrary elementary row transformation of $A_0$, $B_0$, and $C_0$ be $A'_0={P_1}^{-1}A_0$, 
$B'_0={P_2}^{-1}B_0$, and $C'_0={P_3}^{-1}C_0$. $\lambda_1$ is equal to the RLCT of
 $(\|A\circ B\circ C-A'_0\circ B'_0\circ C'_0\|^2, \phi(A, B, C))$.
\end{lemma}

\begin{proof}
Change variables by $A' = {P_1}^{-1}A$. Then,
\begin{align*}
&\sum_{h=1}^{H_0} \left(A_{ih}B_{jh}C_{kh}-(A_0)_{ih}(B_0)_{jh}(C_0)_{kh}\right) \\
=&\sum_{h=1}^{H_0} \left(\left(\sum_{l} (P_1)_{il}A'_{lh}\right) B_{jh}C_{kh} 
- \left(\sum_{l} (P_1)_{il}(A'_0)_{lh}\right) (B_0)_{jh}(C_0)_{kh}\right) \\
=&\sum_{l} (P_1)_{il}\left(\sum_{h=1}^{H_0} (A'_{lh}B_{jh}C_{kh}-(A'_0)_{lh}(B_0)_{jh}(C_0)_{kh})\right).
\end{align*}
Let $T_{ljk}=\sum_{h=1}^{H_0} (A'_{lh}B_{jh}C_{kh}-(A'_0)_{lh}(B_0)_{jh}(C_0)_{kh})$. Since $\left(\sum_{l} (P_1)_{il}T_{ljk}\right)^2$ is quadratic about $\{T_{ljk}\}_{l=1}^{I}$, there exists a positive constant $c$ which satisfies
\begin{equation}
\label{4}
\sum_{i,j,k}\left(\sum_{l} (P_1)_{il}T_{ljk}\right)^2\leq c\|T\|^2.
\end{equation}
Take $\sum_{m} (P_1)_{lm}T_{mjk}$ as $T_{ljk}$ and ${P_1}^{-1}$ as $P_1$ in eq.\eqref{4}. Then, the left side is
\begin{align*}
&\sum_{i,j,k}\left(\sum_{l} ({P_1}^{-1})_{il}\sum_{m} (P_1)_{lm}T_{mjk}\right)^2 \\
=&\sum_{i,j,k} \left(\sum_{m}T_{mjk}\sum_{l} ({P_1}^{-1})_{il}(P_1)_{lm}\right)^2 \\
=&\sum_{i,j,k} \left(\sum_{m}T_{mjk}I_{im}\right)^2 \\
=&\sum_{i,j,k} T_{ijk}^2 \\
=&\|T\|^2,
\end{align*}
where $I$ is an identity matrix. The right side is
\begin{align*}
&c\sum_{l,j,k}\left(\sum_{m} (P_1)_{lm}T_{mjk}\right)^2 \\
=&c\sum_{i,j,k}\left(\sum_{l} (P_1)_{il}T_{ljk}\right)^2.
\end{align*}
So we have $\sum_{i,j,k}\left(\sum_{l} (P_1)_{il}T_{ljk}\right)^2\sim \|T\|^2$. Since $P_1$ is regular, the Jacobian of changing variables by $A' = {P_1}^{-1}A$ does not become $0$ or $\infty$. So it makes no difference to the poles of the zeta function. Using Lemma \ref{lem.1}(a), $\lambda_1$ is equal to the RLCT of $(\|A'\circ B\circ C-A'_0\circ B_0\circ C_0\|^2, \phi(A', B, C))$. Repeating the same operation to $B$ and $C$ completes the proof.
\end{proof}

\begin{lemma}
\label{lem.3}
Let arbitrary elementary column transformation of $A'_0$, $B'_0$, and $C'_0$ be $A''_0=A'_0{Q_1}^{-1}$, $B''_0=B'_0{Q_2}^{-1}$, and $C''_0=C'_0{Q_3}^{-1}$. $\lambda_1$ is smaller than the RLCT of $\left(\sum_{i,j,k}\sum_{n,m,l} (A_{in}B_{jm}C_{kl}-(A''_0)_{in}(B''_0)_{jm}(C''_0)_{kl})^2, \phi(A, B, C)\right)$.
\end{lemma}

\begin{proof}
Change variables by $A'=A{Q_1}^{-1}$, $B'=B{Q_2}^{-1}$, $C'=C{Q_3}^{-1}$. Then
\begin{align*}
&\left(\sum_{h=1}^{H_0} \left(A_{ih}B_{jh}C_{kh}-(A'_0)_{ih}(B'_0)_{jh}(C'_0)_{kh}\right)\right)^2 \\
=&\left(\sum_{h=1}^{H_0}\left(\sum_{n} A'_{in}(Q_1)_{nh} \sum_{m} B'_{jm}(Q_2)_{mh}\sum_{l}C'_{kl}(Q_3)_{lh} \right. \right.\\
&\left. \left. - \sum_{n} (A''_0)_{in}(Q_1)_{nh} \sum_{m} (B''_0)_{jm}(Q_2)_{mh}\sum_{l}(C''_0)_{kl}(Q_3)_{lh}\right)\right)^2 \\
=&\left(\sum_{n,m,l,h} (Q_1)_{nh}(Q_2)_{mh}(Q_3)_{lh}\left(A'_{in}B'_{jm}C'_{kl}
-(A''_0)_{in}(B''_0)_{jm}(C''_0)_{kl}\right)\right)^2.
\end{align*}
Let $D_{nml}=\sum_{h} (Q_1)_{nh}(Q_2)_{mh}(Q_3)_{lh}$, $T_{nml}^{(ijk)}=A'_{in}B'_{jm}C'_{kl}-(A''_0)_{in}(B''_0)_{jm}(C''_0)_{kl}$. Then, there exists a positive constant $c$ which satisfies
\begin{equation*}
\left(\sum_{n,m,l} D_{nml}T_{nml}^{(ijk)}\right)^2
\leq c\sum_{n,m,l} \left(T_{nml}^{(ijk)}\right)^2.
\end{equation*}
Since $Q_1$, $Q_2$, and $Q_3$ are regular, using Lemma \ref{lem.1}(a) completes the proof.
\end{proof}

By lemma \ref{lem.2} and \ref{lem.3}, one of the upper bounds of $\lambda_1$ is the RLCT of \\
$\left(\sum_{i,j,k}\sum_{n,m,l} (A_{in}B_{jm}C_{kl}-(A''_0)_{in}(B''_0)_{jm}(C''_0)_{kl})^2, \phi(A, B, C)\right)$, where $A''_0$, $B''_0$ and $C''_0$ are the diagonalizations of $A_0$, $B_0$, $C_0$. Let us call this RLCT $\lambda'_1$. Let $K'(A, B, C)=\sum_{i,j,k}\sum_{n,m,l} (A_{in}B_{jm}C_{kl}-(A''_0)_{in}(B''_0)_{jm}(C''_0)_{kl})^2$. Let $r_A$, $r_B$, $r_C$ be the rank of $A_0$, $B_0$, $C_0$. Then $A''_0$, $B''_0$ and $C''_0$ have their non-zero elements only until ${r_A}_{th}$, ${r_B}_{th}$ and ${r_C}_{th}$ components. So, we put
\begin{align*}
&G_a=\{(i,n) \mid 1\leq i=n\leq r_A \} \\
&G_b=\{(j,m) \mid 1\leq j=m\leq r_B \} \\
&G_c=\{(k,l) \mid 1\leq k=l\leq r_C \}. \\
\end{align*}
Then,
\begin{align}
&K'(A, B, C) \notag\\
=&\sum_{(i,n)\notin G_a}\sum_{(j,m)\notin G_b}\sum_{(k,l)\notin G_c}(A_{in}B_{jm}C_{kl})^2 \notag\\
+&\sum_{(i,n)\in G_a}\sum_{(j,m)\notin G_b}\sum_{(k,l)\notin G_c}(A_{in}B_{jm}C_{kl})^2 \notag\\
+&\sum_{(i,n)\notin G_a}\sum_{(j,m)\in G_b}\sum_{(k,l)\notin G_c}(A_{in}B_{jm}C_{kl})^2 \notag\\
+&\sum_{(i,n)\notin G_a}\sum_{(j,m)\notin G_b}\sum_{(k,l)\in G_c}(A_{in}B_{jm}C_{kl})^2 \notag\\
+&\sum_{(i,n)\in G_a}\sum_{(j,m)\in G_b}\sum_{(k,l)\notin G_c}(A_{in}B_{jm}C_{kl})^2 \notag\\
+&\sum_{(i,n)\notin G_a}\sum_{(j,m)\in G_b}\sum_{(k,l)\in G_c}(A_{in}B_{jm}C_{kl})^2 \notag\\
+&\sum_{(i,n)\in G_a}\sum_{(j,m)\notin G_b}\sum_{(k,l)\in G_c}(A_{in}B_{jm}C_{kl})^2 \notag\\
 +&\sum_{i=1}^{r_A}\sum_{j=1}^{r_B}\sum_{k=1}^{r_C}(A_{ii}B_{jj}C_{kk}-(A''_0)_{ii}(B''_0)_{jj}(C''_0)_{kk})^2,
\label{5}
\end{align}
which leads to the following Lemma:
\begin{lemma}
\label{lem.6}
Let 
\begin{align*}
&K'_1(A_1, B_1, C_1) \\
=&\sum_{(i,n)\notin G_a}\sum_{(j,m)\notin G_b}\sum_{(k,l)\notin G_c}(A_{in}B_{jm}C_{kl})^2 \\
+&\sum_{(j,m)\notin G_b}\sum_{(k,l)\notin G_c}(B_{jm}C_{kl})^2
+\sum_{(i,n)\notin G_a}\sum_{(k,l)\notin G_c}(A_{in}C_{kl})^2 
+\sum_{(i,n)\notin G_a}\sum_{(j,m)\notin G_b}(A_{in}B_{jm})^2 \\
+&\sum_{(k,l)\notin G_c}(C_{kl})^2
+\sum_{(i,n)\notin G_a}(A_{in})^2
+\sum_{(j,m)\notin G_b}(B_{jm})^2, \\
&K'_2(A_2, B_2, C_2) \\
=&\sum_{i=1}^{r_A}\sum_{j=1}^{r_B}\sum_{k=1}^{r_C}(A_{ii}B_{jj}C_{kk}-(A''_0)_{ii}(B''_0)_{jj}(C''_0)_{kk})^2,
\end{align*}
and
\begin{equation*}
\phi(A, B, C)=\phi_1(A_1, B_1, C_1)\phi_2(A_2, B_2, C_2),
\end{equation*}
where $A_1$, $B_1$, $C_1$ are the components of A, B, C which appear in the formulation of $K'_1$, and $A_2$, $B_2$, $C_2$ are of $K'_2$.
Then, $\lambda'_1 = \lambda'_{1-1}+\lambda'_{1-2}$, where $\lambda'_{1-1}$ is the RLCT of $(K'_1(A_1,B_1,C_1), \phi_1(A_1,B_1,C_1))$ and $\lambda'_{1-2}$ is the RLCT of $(K'_2(A_2,B_2,C_2), \phi_2(A_2,B_2,C_2))$.
\end{lemma}

\begin{proof}
We have to consider only the neighborhood of $K(w)=0$ because if $K(w)>0$, it does not make difference to the pole of the zeta function. When the last term of eq.\eqref{5} is nearly $0$, $B_{jj}$ and $C_{kk}$ is not $0$. So we have $A_{ii}=\frac{(A''_0)_{ii}(B''_0)_{jj}(C''_0)_{kk}+\epsilon}{B_{jj}C_{kk}}$ ,where $\epsilon$ is a value nearly equal to $0$. Since the support of $\phi(A, B, C)$ is compact, there exists a positive number $M$ such that $|A_{ii}|, |B_{jj}|, |C_{kk}|\leq M$. In addition, $(A''_0)_{ii}$, $(B''_0)_{jj}$, $(C''_0)_{kk}$ is not $0$, so $A_{ii}$ is also not 0. Hence there exists a positive number $m$ such that $m\leq |A_{ii}|$. $|B_{jj}|$ and $|C_{kk}|$ also have the positive lower and upper bounds. Then,
\begin{align*}
&\sum_{(i,n)\in G_a}\sum_{(j,m)\notin G_b}\sum_{(k,l)\notin G_c}(A_{in}B_{jm}C_{kl})^2
\sim \sum_{(j,m)\notin G_b}\sum_{(k,l)\notin G_c}(B_{jm}C_{kl})^2 \\
&\sum_{(i,n)\notin G_a}\sum_{(j,m)\in G_b}\sum_{(k,l)\notin G_c}(A_{in}B_{jm}C_{kl})^2
\sim \sum_{(i,n)\notin G_a}\sum_{(k,l)\notin G_c}(A_{in}C_{kl})^2 \\
&\sum_{(i,n)\notin G_a}\sum_{(j,m)\notin G_b}\sum_{(k,l)\in G_c}(A_{in}B_{jm}C_{kl})^2
\sim \sum_{(i,n)\notin G_a}\sum_{(j,m)\notin G_b}(A_{in}B_{jm})^2 \\
&\sum_{(i,n)\in G_a}\sum_{(j,m)\in G_b}\sum_{(k,l)\notin G_c}(A_{in}B_{jm}C_{kl})^2
\sim \sum_{(k,l)\notin G_c}(C_{kl})^2 \\
&\sum_{(i,n)\notin G_a}\sum_{(j,m)\in G_b}\sum_{(k,l)\in G_c}(A_{in}B_{jm}C_{kl})^2
\sim \sum_{(i,n)\notin G_a}(A_{in})^2 \\
&\sum_{(i,n)\in G_a}\sum_{(j,m)\notin G_b}\sum_{(k,l)\in G_c}(A_{in}B_{jm}C_{kl})^2
\sim \sum_{(j,m)\notin G_b}(B_{jm})^2.
\end{align*}
By Lemma \ref{lem.1}(a), the RLCT of $(K'(A, B, C), \phi(A, B, C))$ is equal to that of $(K'_1(A_1, B_1, C_1)+K'_2(A_2, B_2, C_2), \phi(A, B, C))$. The variables are independent between $K'_1(A_1, B_1, C_1)$ and $K'_2(A_2, B_2, C_2)$, so Lemma \ref{lem.1}(c) completes the proof.
\end{proof}

The next lemma is from \cite{19}.
\begin{lemma}
\label{lem.9}
Let $x_{i}, y_{j}$ be variables whose absolute values are finite $(i,j\in\mathbb{N})$. Let $a_{i}, b_{j}$ be constant values. Let $f_{ij}=x_{i}y_{j}-a_{i}b_{j}$. Then, $\forall M, N \geq 2 \in \mathbb{N}$, 
\begin{equation*}
\sum_{i=1}^M\sum_{j=1}^N f_{ij}^2 \sim \sum_{i=2}^M f_{i1}^2 + \sum_{j=2}^N f_{1j}^2 + f_{11}^2.
\end{equation*}
\end{lemma}

Using the above lemma, the following corollary is easily attained:
\begin{corollary}
\label{cor.1}
Let $x_{i}, y_{j}, z_{k}$ be variables whose absolute values are finite $(i,j,k\in\mathbb{N})$. Let $a_{i}, b_{j}, c_{k}$ be constant values. Let $f_{ijk}=x_{i}y_{j}z_{k}-a_{i}b_{j}c_{k}$. Then, $\forall I, J, K \geq 2 \in \mathbb{N}$, 
\begin{equation*}
\sum_{i=1}^I\sum_{j=1}^J\sum_{k=1}^K f_{ijk}^2\sim \sum_{i=2}^I f_{i11}^2+\sum_{j=2}^J f_{1j1}^2+\sum_{k=2}^K f_{11k}^2+f_{111}^2.
\end{equation*}
\end{corollary}

\begin{proof}
Put $y'_{(j-1)K+k}=y_{j}z_{k}$, $b'_{(j-1)K+k}=b_{j}c_{k}$, and $g_{il}=x_{i}y'_{l}-a_{i}b'_{l}$. Then $f_{ijk}=g_{il}$. Using $g_{il}$ as $f_{ij}$ in Lemma \ref{lem.9},
\begin{align*}
\sum_{i=1}^I\sum_{l=1}^{JK} g_{il}^2\sim &\sum_{i=2}^I g_{i1}^2+\sum_{l=2}^{JK} g_{1l}^2+g_{11}^2 \\
=&\sum_{i=2}^I f_{i11}^2+\sum_{2\leq(j-1)K+k\leq JK} f_{1jk}^2+f_{111}^2 \\
=&\sum_{i=2}^I f_{i11}^2+\sum_{j=1}^J\sum_{k=1}^K f_{1jk}^2.
\end{align*}
$f_{1jk}=x_{1}y_{j}z_{k}$, so putting $x'_{k}=x_{1}z_{k}$, $a'_{k}=a_{1}c_{k}$, $h_{kj}=x'_{k}y_{j}-a'_{k}b_{j}$, we obtain $f_{1jk}=h_{kj}$. Using $h_{kj}$ as $f_{ij}$ in Lemma \ref{lem.9}, 
\begin{equation*}
\sum_{k=1}^K\sum_{j=1}^J h_{kj}^2\sim \sum_{k=2}^K h_{k1}^2+\sum_{j=2}^J h_{1j}^2+h_{11}^2
\end{equation*}
holds. Hence, $\sum_{j=1}^J\sum_{k=1}^K f_{1jk}^2\sim \sum_{j=2}^J f_{1j1}^2+\sum_{k=2}^K f_{11k}^2+f_{111}^2$, which completes the proof.
\end{proof}

\subsection{Preparing for Upper Bound of $\lambda_2$}
\label{section2}

Second, let us prepare for obtaining the upper bound of $\lambda_2$ by the following two lemmas.
\begin{lemma}
\label{lem.7}
If the interval of the integration in the zeta function becomes smaller, the RLCT becomes bigger.
\end{lemma}

The next lemma is from the result of \cite{9}.
\begin{lemma}
\label{lem.8}
Let $B\in\mathbb{R}^{NH}$ and $A\in\mathbb{R}^{HM}$. Then, the RLCT of $(\|BA\|^2, \phi(A, B))$ is
\begin{align*}
&F(N, M, H)= \\
&\begin{cases}
{\frac{2(NM+MH+HN)-(N^2+M^2+H^2)}{8} (|N-M|\leq H\leq N+M\,and\,H+M-N\,is\,even)} \\
{\frac{2(NM+MH+HN)-(N^2+M^2+H^2)+1}{8} (|N-M|\leq H\leq N+M\,and\,H+M-N\,is\,odd)} \\
{\frac{MH}{2}(H\leq N-M)} \\
{\frac{NH}{2}(H\leq M-N)} \\
{\frac{NM}{2}(M+N\leq H)}, \\
\end{cases}
\end{align*}
where $\|\cdot\|$ is the Frobenius norm.
\end{lemma}

\subsection{Proof of Main Theorem}
\label{section3}
\begin{proof}
We prove the main theorem by 3 parts. \\
(step.1) Calculating $\lambda'_{1-1}$\\
$\lambda'_{1-1}$ is obtained by the blowing up of $K'_1$ with the center \cite{6,7,20}: $\{A_1 = B_1 = C_1 = 0\}$. From the symmetry of the variables, we only have to consider the following case:
$$\begin{cases}
	{a_{11}^{(1)} = \nu}\\
	{a_{ij}^{(1)} = \nu a_{ij}^{(1)} ((i,j)\neq(1,1))} \\
	{B_1 = \nu B_1, C_1 = \nu C_1},
    \end{cases}
$$
where $a_{11}^{(1)}$ is a component of $A_1$. Hence we get
\begin{equation*} 
K'_1(A_1, B_1, C_1)=\nu^2(1 + (positive\,variable)),
\end{equation*}
which is called \textit{normal crossing}. The Jacobian of this transformation is 
\begin{equation*}
\nu^{(I+J+K)H - (r_A+r_B+r_C) - 1},
\end{equation*}
so the RLCT of $(K'_1(A_1, B_1, C_1), \phi_1(A_1, B_1, C_1))$ is,
\begin{equation*} 
\lambda'_{1-1} = \frac{(I+J+K)H - (r_A+r_B+r_C)}{2}.
\end{equation*} \\
(step.2) Calculating $\lambda'_{1-2}$ \\
$\lambda'_{1-2}$ is obtained by use of Corollary \ref{cor.1}. Consider using $A_{ii}$, $B_{jj}$, $C_{kk}$,$(A''_{0} )_{ii}$, $(B''_{0} )_{jj}$, $(C''_{0} )_{kk}$ as $x_{i}$, $y_{j}$, $z_{k}$, $a_{i}$, $b_{j}$, $c_{k}$ in the Corollary \ref{cor.1} respectively. Then 
\begin{align*}
K'_2(A_2, B_2, C_2)&\sim 
\sum_{i=2}^{r_A} \left(A_{ii}B_{11}C_{11}-(A''_{0})_{ii}(B''_{0})_{11}(C''_{0})_{11}\right)^2 \\
&+\sum_{j=2}^{r_B} \left(A_{11}B_{jj}C_{11}-(A''_{0})_{11}(B''_{0})_{jj}(C''_{0})_{11}\right)^2 \\
&+\sum_{k=2}^{r_C} \left(A_{11}B_{11}C_{kk}-(A''_{0})_{11}(B''_{0})_{11}(C''_{0})_{kk}\right)^2 \\
&+\left(A_{11}B_{11}C_{11}-(A''_{0})_{11}(B''_{0})_{11}(C''_{0})_{11}\right)^2
\end{align*}
holds. Here, consider the following transformation:
$$\begin{cases}
	{A'_{ijk}=A_{ii}B_{jj}C_{kk}-(A''_{0})_{ii}(B''_{0})_{jj}(C''_{0})_{kk}} \\
	{B'_{j}=B_{jj}}\\
	{C'_{k}=C_{kk}}.
    \end{cases}
$$
Then
\begin{equation}
\label{8}
K'_2(A_2, B_2, C_2)\sim \sum_{i=2}^{r_A} (A'_{i11})^2+\sum_{j=2}^{r_B} (A'_{1j1})^2+\sum_{k=2}^{r_C} (A'_{11k})^2+(A'_{111})^2
\end{equation}
holds. The Jacobian of the transformation is
\begin{equation*}
\left| \frac{\partial (A'_{ijk},B'_{j},C'_{k})}{\partial ({A_{ii},B_{jj},C_{kk})}}\right|
=
\begin{vmatrix}
B_{jj}C_{kk}&A_{ii}C_{kk}&A_{ii}B_{jj} \\
0&1&0 \\
0&0&1 
\end{vmatrix}
=B_{jj}C_{kk}\neq0, \infty,
\end{equation*}
so it makes no difference to the pole of the zeta function. Since the variables are independent among the terms of the right-side of eq.\eqref{8} and the prior distribution $\phi(A, B, C)$ also makes no difference to the pole, we can use Lemma \ref{lem.1}(c). The RLCT of $K(w)=w^2$ is $\frac{1}{2}$, so we get
\begin{equation*}
\lambda'_{1-2}=\frac{r_A+r_B+r_C-2}{2}.
\end{equation*} \\
(step.3) Calculating the upper bound of $\lambda_2$ \\
From Lemma \ref{lem.7}, the RLCT of $\left(\int \left(\sum_{i,j,k} \left(\sum_{h=H_0+1}^H A_{ih}B_{jh}C_{kh}\right)^2\right)dw, \phi(A,B,C)\right)$ is less than that of $\left(\int_{B's\,components\neq0} \left(\sum_{i,j,k} \left(\sum_{h=H_0+1}^H A_{ih}B_{jh}C_{kh}\right)^2\right)dw,\phi(A, B, C)\right)$. Under the condition in which all $B$'s components are not 0, consider the following transformation:
$$\begin{cases}
	{A'_{i(J-1)+j,h}=A_{ih}B_{jh}} \\
	{B'_{jh}=B_{jh}}\\
	{C'_{kh}=C_{kh}}.
    \end{cases}
$$
The Jacobian of the transformation is 
\begin{equation*}
\left|\frac{\partial (A'_{i(J-1)+j,h},B'_{jh},C'_{kh})}{\partial (A_{ih},B_{jh},C_{kh})}\right|
=
\begin{vmatrix}
B_{jh}&A_{ih}&0 \\
0&1&0 \\
0&0&1 \\
\end{vmatrix}
=B_{jh}\neq0,
\end{equation*}
so it makes no difference to the pole of the zeta function. Let us calculate the RLCT of $\left(\sum_{i=1}^{IJ} \sum_{k=1}^K \left(\sum_{h=H_0+1}^{H} A'_{ih}C'_{kh}\right)^2=\|A'C'^T\|, \phi(A, B, C)\right)$, where $T$ means the transpose. From Lemma \ref{lem.8}, it is $F(IJ, K, H-H_0)$. We remark that it holds even when $H-H_0=0$ because $F(N, M, 0)=0$. From the symmetry, $F(JK, I, H-H_0)$ and $F(KI, J, H-H_0)$ are also the upper bound of $\lambda_2$. So the upper bound of $\lambda_2$ is 
\begin{equation*}
\min\{F(IJ, K, H-H_0), F(JK, I, H-H_0), F(KI, J, H-H_0)\}.
\end{equation*}
Let $m_1=F(IJ, K, H-H_0)$, $m_2=F(JK, I, H-H_0)$ and $m_3=F(KI, J, H-H_0)$. From (step.1)-(step.3), the upper bound of $\lambda$ is
\begin{align*}
\lambda \leq \lambda_1+\lambda_2 \leq &\lambda'_{1-1}+\lambda'_{1-2}+\min\{m_1, m_2, m_3\} \\
=& \frac{H_0(I+J+K)-2}{2}+\min\{m_1, m_2, m_3\}.
\end{align*}

\end{proof}

\section{Discussion}
\subsection{Numerical Experiments and Tightness of the Bound}
Let us examine the exact value of the RLCT of the tensor decomposition and compare it with the bound in the main theorem. We make experiments of calculating the approximate value of $\mathbb{E}[G_n]$. Concretely, 
\begin{align*}
&\mathbb{E}[G_n] \\
=&\mathbb{E}_{X\sim p(X|A_0,B_0,C_0)}\left[\mathbb{E}_{x\sim p(x|A_0,B_0,C_0)}\left[\log\frac{p(x|A_0,B_0,C_0)}{\mathbb{E}_{(A,B,C)\sim p(A, B, C|X^n)}\left[p(x|A,B,C)\right]}\right]\right]
\end{align*}
holds, and the first and second expectation can be approximately calculated by sampling the components of $X$ and $x$ from $\mathcal{N}((A_0\circ B_0\circ C_0)_{ijk}, 1)$ respectively, where $\mathcal{N}$ denotes the normal distribution, and the third expectation by sampling $(A,B,C)$ from the posterior distribution. The posterior distribution is made by the Markov chain Monte Carlo method. We sample 10 data for the first expectation, 10000 for the second, and 1000 for the third. The approximate value of $\lambda$ can be attained by use of the following corollary:
\begin{corollary}
\label{cor.2}
If there is a parameter $w_0$ which satisfies
\begin{equation*}
q(x)=p(x|w_0),
\end{equation*}
then $S=0$ in eq.\eqref{2}, i.e., 
\begin{equation*}
\mathbb{E}[G_n] =  \frac{\lambda}{n} + o(\frac{1}{n}).
\end{equation*}
\end{corollary}
Using Corollary \ref{cor.2}, we set $n=100$ and approximately calculate $\lambda$ by $n\mathbb{E}[G_n]$. We experiment under the condition $I=J=K=2,3,4$ and $H=2H_0=2,4,6,8,10$. For each case, we randomly change the true parameters $(A_0,B_0, C_0)$ for 5 times and calculate the errors.

The result is shown in Table \ref{table.1}. Let us denote the upper bound as $\lambda_B$. When $I=J=K=2$ or $I=J=K=3$, our bound $\lambda_B$ is very loose when $H$ is large: the exact value $\lambda$ is unexpectedly small. This result can be explained by the number of the poles of the zeta function. As $H$ is getting bigger, the number of the combinations $(A, B, C)$ which satisfies $A\circ B \circ C=A_0\circ B_0 \circ C_0$ is also getting bigger, which means the number of the poles increases. These facts result in the RLCT decreasing. Moreover, this phenomenon becomes hard to happen when $I=J=K$ is getting large. This is simply because $A\circ B \circ C=A_0\circ B_0\circ C_0$ is difficult to be satisfied when there are many components of the tensor $A\circ B \circ C$.
\begin{table}[hbtp]
\caption{Comparison of Numerical Calculated Values and the Theoretical Upper Bounds ($H=2H_0$).}
\label{table.1}
\begin{tabular}{|c|c|c|c|c|c|c|}\hline
\multicolumn{2}{|c|}{} & $H_0=1$ & $H_0=2$ & $H_0=3$ & $H_0=4$ & $H_0=5$ \\ \hline
\multirow{2}{*}{I=J=K=2} &$\lambda_B$ & 3.00 & 7.00 & 11.00 & 14.50 & 18.00 \\ \cline{2-7}
 &$\lambda$ & $2.68\pm{0.57}$ & $4.31\pm{0.58}$ & $3.82\pm{0.84}$ & $3.76\pm{0.39}$ & $4.06\pm{0.59}$ \\ \hline
 \multirow{2}{*}{I=J=K=3} &$\lambda_B$ & 5.00 & 11.00 & 17.00 & 23.00 & 29.00 \\ \cline{2-7}
 &$\lambda$ & $5.00\pm{1.54}$ & $9.73\pm{1.00}$ & $13.49\pm{2.20}$ & $12.06\pm{0.75}$ & $12.68\pm{2.14}$ \\ \hline
 \multirow{2}{*}{I=J=K=4} &$\lambda_B$ & 7.00 & 15.00 & 23.00 & 31.00 & 39.00 \\ \cline{2-7}
 &$\lambda$ & $6.50\pm{1.11}$ & $13.66\pm{0.81}$ & $20.13\pm{1.98}$ & $29.43\pm{2.25}$ & $32.66\pm{2.08}$ \\ \hline
\end{tabular}
\end{table}

\subsection{Future Study}
One of the future directions of this research is to find tighter bounds. 
From the numerical experiment, we find that the upper bound is not tight enough especially when $H$ is large and there are a lot of poles in the zeta function. Our upper bound of $\lambda_1$ is $\frac{H_0(I+J+K)-2}{2}$, whereas the obvious upper bound is $\frac{H_0(I+J+K)}{2}$.

Another direction is to generalize the assumptions. We made many assumptions in section.2 for the main theorem, so considering generalized conditions is also important.

\section{Conclusion}
In this paper, we derived an upper bound of the RLCT of the tensor decomposition model, which shows the Bayesian generalization error, and examined its tightness by comparing them with the numerical 
experiments.

\section*{Acknowledgements} This work was partially supported by JSPS KAKENHI
Grant-in-Aid for Scientific Research (C) 21K12025.

\end{document}